%% file: rank-1.tex
\documentclass{article}

\newif\ifneurips\neuripsfalse

\ifneurips
\usepackage{neurips_2021}
\renewcommand{\paragraph}[1]{\textbf{#1}}
\else

\title{\textsc{Bandit Phase Retrieval}}
\author{Tor Lattimore and Botao Hao \\ DeepMind, London \\ \texttt{\{lattimore,bhao\}@deepmind.com}}
\date{}
\fi

\usepackage[utf8]{inputenc} 
\usepackage[T1]{fontenc}    
\usepackage{comment}
\usepackage{url}            
\usepackage{booktabs}       
\usepackage{amsfonts}       
\usepackage{nicefrac}       
\usepackage{microtype}      
\usepackage{mathtools}

\newif\ifsup\suptrue

\usepackage{tcolorbox}
\usepackage[textsize=tiny]{todonotes}
\newcommand{\todot}[2][]{\todo[size=\scriptsize,color=red!20!white,#1]{Tor: #2}}

\usepackage{xcolor}
\usepackage{amsmath}
\usepackage{amsthm}
\usepackage{enumitem}
\usepackage{listings}
\usepackage{amssymb}
\usepackage[boxed]{algorithm}
\usepackage{bm}

\definecolor{dkblue}{cmyk}{1,.54,.04,.19}

\usepackage{hyperref} 

\hypersetup{
  bookmarks=true,         
    unicode=false,          
    pdftoolbar=true,        
    pdfmenubar=true,        
    pdffitwindow=false,     
    pdfstartview={FitH},    
    pdftitle={Bandit Phase Retrieval},    
    pdfauthor={Tor Lattimore},     
    pdfsubject={Bandits},   
    pdfcreator={pdflatex},   
    pdfproducer={Producer}, 
    pdfkeywords={bandits} {online learning} {machine learning}, 
    pdfnewwindow=true,      
    colorlinks=true,        
    linkcolor=dkblue,       
    citecolor=dkblue,       
    filecolor=dkblue,       
    urlcolor=dkblue,        
}

\usepackage[capitalise]{cleveref}
\usepackage[bf]{caption}
\theoremstyle{plain}
\newtheorem{theorem}{Theorem}

\newtheorem{lemma}[theorem]{Lemma}

\newtheorem{corollary}[theorem]{Corollary}

\newcommand{\ceil}[1]{\left\lceil #1 \right\rceil}
\newcommand{\floor}[1]{\left\lfloor #1 \right\rfloor}
\theoremstyle{definition}

\theoremstyle{remark}

\usepackage{natbib}

\newcommand{\R}{\mathbb R}

\newcommand{\cA}{\mathcal A}
\newcommand{\cD}{\mathcal D}
\newcommand{\cL}{\mathcal L}
\newcommand{\cH}{\mathcal H}
\newcommand{\argmin}{\operatornamewithlimits{arg\,min}}
\newcommand{\argmax}{\operatornamewithlimits{arg\,max}}
\newcommand{\ip}[1]{\langle #1 \rangle}
\newcommand{\bip}[1]{\left\langle #1 \right\rangle}
\newcommand{\Reg}{\mathfrak{R}}
\newcommand{\norm}[1]{\Vert #1 \Vert}

\newcommand{\E}{\mathbb E}
\newcommand{\sind}{\bm 1}
\newcommand{\cC}{\mathcal C}

\newcommand{\cN}{\mathcal N}
\newcommand{\cX}{\mathcal X}
\newcommand{\cY}{\mathcal Y}

\newcommand{\cE}{\mathcal E}
\newcommand{\cF}{\mathcal F}

\newcommand{\warm}{\widehat A_w}

\newcommand{\zeros}{ \bm 0}
\newcommand{\bbP}{\mathbb P}
\newcommand{\laspan}{\operatorname{span}}
\newcommand{\bbQ}{\mathbb Q}
\newcommand{\ball}{\mathbb B}

\newcommand{\const}{\operatorname{const}}
\newcommand{\KL}{\operatorname{KL}}
\newcommand{\reg}{\mathfrak{r}}
\newcommand{\sphere}{\mathbb S}
\newcommand{\Var}{\mathbb V}
\renewcommand{\d}[1]{\operatorname{d}\!#1}

\ifneurips
\title{Bandit Phase Retrieval}
\author{%
  Tor Lattimore\\
  DeepMind, London \\
  \texttt{lattimore@deepmind.com} \\
   \And
   Botao Hao\\
   DeepMind, London \\
   \texttt{haobotao000@gmail.com} \\
}
\fi

\begin{document}

\maketitle

\begin{abstract}
We study a bandit version of phase retrieval where the learner chooses actions $(A_t)_{t=1}^n$ in the $d$-dimensional unit ball and the expected reward is $\ip{A_t, \theta_\star}^2$ where $\theta_\star \in \R^d$ is an unknown parameter vector. We prove that the minimax cumulative regret in this problem is $\smash{\tilde \Theta(d \sqrt{n})}$, which improves on the best known bounds by a factor of $\smash{\sqrt{d}}$. We also show that the minimax simple regret is $\smash{\tilde \Theta(d / \sqrt{n})}$ and that this is only achievable by an adaptive algorithm. Our analysis shows that an apparently convincing heuristic for guessing lower bounds can be misleading and
that uniform bounds on the information ratio for information-directed sampling \citep{RV14} are not sufficient for optimal regret. 
\end{abstract}

\section{Introduction}
We study an instantiation of the low-rank bandit problem \citep{JWW19} that in the statistical setting
is called phase retrieval. Although this model is interesting in its own right,
our main focus is on the curious information structure of this problem and how it
impacts algorithm design choices. 
Notably, we were not able to prove optimal regret for standard approaches based on optimism, Thompson sampling or even information-directed sampling.
Instead, our algorithm is a variant of explore-then-commit with an adaptive exploration phase that learns to gain information at a faster rate than 
what is achievable with non-adaptive exploration.

\paragraph{Problem setting}
Let $\norm{\cdot}$ be the standard euclidean norm and $\ball^d_r = \{x \in \R^d : \norm{x} \leq r\}$ and $\sphere^{d-1}_r = \{x \in \R^d : \norm{x} = r\}$.
At the start of the game the environment secretly chooses a vector $\theta_\star \in \sphere^{d-1}_r$ with $r \in [0, 1]$ a constant that is known to the learner. The assumption that $r$ is known can be relaxed at no cost (\cref{sec:disc}).
The game then proceeds over $n$ rounds. In round $t$ the learner chooses an action $A_t \in \ball^d_1$ and observes a reward
\begin{align*}
X_t = \ip{A_t, \theta_\star}^2 + \eta_t\,,
\end{align*}
where $(\eta_t)_{t=1}^n$ is a sequence of independent standard Gaussian random variables. As is standard in bandit problems, 
the conditional law of $A_t$ should be chosen as a (measurable) function
of the previous actions $(A_s)_{s=1}^{t-1}$ and rewards $(X_s)_{s=1}^{t-1}$.
The performance of a policy $\pi$ is measured in terms of the expected regret,
\begin{align*}
\Reg_n(\pi, \theta_\star) &= \max_{a \in \ball^d_1} \E\left[\sum_{t=1}^n \left(\ip{a, \theta_\star}^2 - \ip{A_t, \theta_\star}^2\right)\right]
= r^2 - \E\left[\sum_{t=1}^n \ip{A_t, \theta_\star}^2\right]\,. 
\end{align*}
The minimax regret is $\Reg_n^\star = \sup_{r \in [0,1]} \inf_\pi \sup_{\theta_\star \in \sphere^{d-1}_r} \Reg_n(\pi, \theta_\star)$, where the infimum is over all policies.

We also study the pure exploration setting, where at the end of the game the learner uses the observed data $(A_t)_{t=1}^n$ and $(X_t)_{t=1}^n$ to
make a prediction $\widehat A_\star \in \ball^d_1$ of the optimal action. The simple regret of policy $\pi$ is
\begin{align*}
    \reg_n(\pi, \theta_\star) = \max_{a \in \ball^d_1} \E\left[\ip{a, \theta_\star}^2 - \ip{\widehat A_\star, \theta_\star}^2\right]
    = r^2 - \E\left[\ip{\widehat A_\star, \theta_\star}^2\right]\,.
\end{align*}
As expected, the minimax simple regret is
$\reg_n^\star = \sup_{r \in [0,1]} \inf_{\pi} \sup_{\theta_\star \in \sphere_r^{d-1}} \reg_n(\pi, \theta_\star)$.

\paragraph{Contributions}
Our main contribution is nearly matching upper and lower bounds on $\Reg_n^\star$.
For the simple regret we provide a near-optimal upper bound and a lower bound showing
that non-adaptive policies must be at least a factor of $\Omega(\sqrt{d})$ suboptimal.
In all of the following, $\const$ is a universal non-negative constant that may vary from one expression to the next.

\begin{theorem}\label{thm:main}
$\Reg_n^\star \leq \const d \sqrt{n \log(n) \log(d)}$.
\end{theorem}

\begin{theorem}\label{thm:lower}
$\Reg_n^\star \geq \const d \sqrt{n}$
whenever $d$ is larger than a suitable universal constant and $n$ is sufficiently large. 
\end{theorem}

\begin{theorem}\label{thm:simple-upper}
$\reg_n^\star \leq \const d \sqrt{\log(n) \log(d) / n}$.
\end{theorem}

\begin{theorem}\label{thm:simple-lower}
Assume that $n \geq d \geq 8$. Then there exists an $r \in [0,1]$ such that for all policies $\pi$ with $(A_t)_{t=1}^n$ independent of $(X_t)_{t=1}^n$,
   $\sup_{\theta_\star \in \sphere^{d-1}_r} \reg_n(\pi, \theta_\star) \geq \const \sqrt{d^3/n}$.
\end{theorem}

We also show that worst-case bounds on the information ratio for information-directed sampling are \textit{not} sufficient to achieve optimal
regret.
Our results suggest that the conjectured lower bounds for low-rank bandits \citep{JWW19,LMT21} are not true and that existing upper bounds are loose.
The same phenomenon may explain the gap between upper and lower bounds for bandit principle component analysis \citep{KN19}, as we discuss in \cref{sec:disc}.

\paragraph{Notation}
The first $n$ integers are $[n] = \{1,2,\ldots,n\}$ and the standard basis vectors in $\R^d$ are $e_1,\ldots,e_d$. 
The span of a collection of vectors is denoted by $\laspan(v_1,\ldots,v_m)$ and the orthogonal complement of a linear subspace $V \subset \R^d$ 
is $V^\perp = \{x \in \R^d : \ip{x, y} = 0 \text{ for all } y \in V\}$.
The mutual information between random elements $X$ and $Y$ on the same probability space is $I(X; Y)$ and the relative entropy between probability measures $P$ and $Q$ on the same measurable space is $\KL(P, Q)$.
The dimension of a set $\Theta \subset \R^d$ is defined as the dimension of the affine hull of $\Theta$.

\section{Related work}\label{sec:related}

\paragraph{Phase retrieval}
Phase retrieval is a classical problem in signal processing and statistics \citep{candes2015phase, candes2015phase1, cai2016optimal, chen2017solving, chen2019gradient, sun2018geometric}. These works are focused
on learning $\theta_\star$ where the covariates $(A_t)_{t=1}^n$ are 
uncontrolled, either random or fixed design.

\paragraph{Linear bandits}
Our problem can be written as a stochastic linear bandit by noticing that $\ip{A_t, \theta_\star}^2 = \ip{A_t A_t^\top, \theta_\star \theta_\star^\top}$, where
the inner product between matrices on the right-hand side should be interpreted coordinate-wise and the
action set is $\{aa^\top : a \in \ball^d_1\}$. 
There is an enormous literature on stochastic linear bandits \citep{auer2002using, dani2008stochastic, rusmevichientong2010linearly, chu2011contextual, abbasi2011improved}. 
This reduction immediately yields an upper bound on the minimax regret of $O(d^2 \sqrt{n} \log(n))$.

\paragraph{Low-rank bandits}
Low-rank bandits are a kind of linear bandit where the environment is determined by an unknown matrix and the actions of the learner are also matrices.
Let $\cE \subset \R^{d_1 \times d_2}$ and $\cA \subset \R^{d_1 \times d_2}$. 
A low-rank bandit problem over $\cE$ and with actions $\cA$ is characterised by a matrix $\Theta_\star \in \cE$. The learner plays actions $A_t \in \cA$ and the reward is $X_t = \ip{A_t, \Theta_\star} + \eta_t$, where $\eta_t$ is noise and the inner product between matrices is interpreted coordinate-wise. So far this is nothing more than a complicated way of defining a linear bandit. The name comes from the fact that in general elements of $\cE$ are assumed to be low-rank.
The precise nature of the problem is determined by assumptions on $\cE$ and the action set $\cA$. Our setup is recovered by assuming that $\cE = \{xx^\top : x \in \sphere^{d-1}_r\}$ and $\cA = \{x x^\top : x \in \ball_1^d\}$.

\cite{JWW19} assume that $\cE$ consists of rank $p$ matrices and $\cA = \{x y^\top : x \in \cX, y \in \cY\}$ for some reasonably bounded sets $\cX \subset \R^{d_1}$ and $\cY \subset \R^{d_2}$. They prove that the regret is bounded by $\tilde O((d_1+d_2)^{3/2} \sqrt{p n})$. These results cannot be applied directly to the phase retrieval bandit because of the product assumption on the action set.
\cite{LMT21} retain the assumption that $\cE$ consists of rank-$p$ matrices, but relax the product form of the action set (while also allowing for generalised linear models). Relying only on mild boundedness assumptions, they show that the regret can be bounded by $\tilde O((d_1 + d_2)^{3/2} \sqrt{pn})$. For the bandit phase retrieval problem, $d_1 = d_2 = d$ and $p = 1$, so this algorithm yields an upper bound on the regret for bandit phase retrieval of $\tilde O(d^{3/2} \sqrt{n})$. Both \cite{JWW19} and \cite{LMT21} conjecture that their upper bounds are optimal. Our results show that this is not true 
for this sub-problem, despite the fact that the heuristic argument used by these authors 
holds in this case, as we explain in \cref{sec:info}. We summarize these comparisons in Table \ref{table:compar}.

Some authors use a model where the noise is in the parameter rather than additive, which means the reward is $\ip{A_t, \Theta_t}$ with $(\Theta_t)_{t=1}^n$ an independent and identically distributed sequence of low-rank matrices (with unknown distribution). For example, \cite{katariya2017stochastic, katariya2017bernoulli} and \cite{trinh2019solving} assume that $\Theta_t$ is rank-$1$ almost surely and $\cA = \{ e_i e_j^\top : 1 \leq i,j \leq d\}$, which means the learner is trying to identify the largest entry in a matrix.

\paragraph{Adversarial setting}
A similar problem has been studied in the adversarial framework by \cite{KN19}.
They assume that $(\theta_t)_{t=1}^n$ is a sequence of vectors chosen in secret by an adversary at the start of the
game and the learner observes $\ip{A_t, \theta_t}^2$.
They design an algorithm for which the regret is at most $\tilde O(d \sqrt{n})$, while the best lower bound is $\Omega(\sqrt{d n})$.

\begin{table}
\centering
\caption{For the low-rank bandits column, $p$ is the rank. We ignore logarithmic factors and universal constant. Note, the $dp\sqrt{n}$ lower bound derived by \cite{LMT21} does not apply to bandit phase retrieval because it makes use of the richer structure of the more general model.}
\label{table:compar}
\scalebox{0.75}{
\renewcommand{\arraystretch}{1.2}
\begin{tabular}{|l|c|c|c|} 
 \hline \textbf{Upper bounds} & Bandit phase retrieval& Low-rank bandits & Pure exploration\\ 
 \hline
\cite{abbasi2011improved} & $O(d^2\sqrt{n})$ & $O(d^2\sqrt{n})$ &  N/A \\
  \hline
\cite{JWW19, LMT21}
  &$O(d^{3/2}\sqrt{n})$ & $O(d^{3/2}\sqrt{pn})$ & N/A\\
 \hline
 This work
  &$O(d\sqrt{n})$ & N/A & $O(d/\sqrt{n})$\\
  \hline 
  \textbf{Lower bounds} & & & \\ 
   \hline 
\cite{LMT21}
  &N/A& $\Omega(dp\sqrt{n})$ & N/A\\
 \hline
   This work
  &$\Omega(d\sqrt{n})$ &$\Omega(d\sqrt{n})$ & $\Omega(d/\sqrt{n})$\\
   \hline
   This work (non-adaptive learning)
  &N/A &N/A & $\Omega(d^{3/2}/\sqrt{n})$\\
  \hline
\end{tabular}}

\end{table}

\section{Information-theoretic heuristics and information-directed sampling}\label{sec:info}

\citeauthor{JWW19} [\citeyear{JWW19}, \S5] argue by comparing the signal to noise ratios between linear and low-rank bandits that the minimax regret for problems like bandit phase retrieval should be lower bounded by $\Omega(d^{1.5} \sqrt{n})$. We make this argument a little more formal and explain why it does not yield the right answer in this instance.
Suppose that $\theta_\star$ is sampled uniformly from $\mathbb S_r^{d-1}$ and the learner takes an action $A \in \mathbb B_1^d$ and observes $X = \ip{A, \theta_\star}^2 + \eta$ with
$\eta\sim N(0,1)$. What is the information gained by the learner?
A symmetry argument shows that all actions on the unit sphere have the same information gain, so let's just fix some $A \in \mathbb S_1^{d-1}$. A Taylor series expansion yields the approximation
\begin{align}
I(\theta_\star ; Y) = \E[\KL(\bbP_{X |\theta_{\star}}, \bbP_X)]
&\approx \frac{1}{2} \E\left[(\E[X | \theta_{\star}] - \E[X])^2\right] \nonumber \\
&= \frac{r^4}{2} \left(\frac{3}{d^2 + 2d} - \frac{1}{d^2}\right)
\leq \frac{r^4}{d^2}\,.
\label{eqn:information_gain}
\end{align}
Note that $\frac{d}{2} \log(n)$ bits are needed to code $\theta_\star$ to reasonable accuracy. So if we presume 
that the rate of information learned
\emph{stays the same} throughout the learning process, then over $n$ rounds the learner can only obtain $O(n r^4 / d^2)$ bits by Eq.~\eqref{eqn:information_gain}.
By setting $r^2 = d^{3/2} \sqrt{\log(n) / n}$ one could be led to believe that the learner 
cannot identify the optimal direction
and the regret would be $\Omega(n r^2) = \Omega(d^{3/2} \sqrt{n \log(n)})$.
The main point is that the rate of information accumulated by a careful learner \emph{increases} over time.

\paragraph{Information-directed sampling}
Our upper bound and the observation above has an important implication.
Suppose as above that $\theta_\star$ is sampled uniformly from $\sphere^{d-1}_r$ and let $A$ be a possibly randomised action.
Since the learner cannot know the realisation of $\theta_\star$ initially, her expected regret for any action 
$a \in \mathbb B^d_1$ is $\Delta(a) = r^2 - \E[\ip{a, \theta_\star}^2] = r^2(1 - \norm{a}^2/d) \geq r^2(1 - 1/d)$. 
On the other hand, as outlined above, the information gain about $\theta_\star$ is about
$r^4 / d^2$. Together these results show that the information ratio is bounded by
\begin{align*}
\Psi \triangleq \frac{\E[\Delta(A)]^2}{I(\theta_\star ; X, A)} 
= \Theta(d^2)\,.
\end{align*}
Since the entropy of a suitable approximation of the optimal action is about $\frac{d}{2} \log(n)$, an application
of the information theoretic analysis by \cite{RV14} suggests that the Bayesian regret can be bounded by
$O(d^{3/2} \sqrt{n \log(n)})$, 
which is suboptimal. This time the problem is that we have used the worst-case bound on the information ratio, without
taking into account the possibility that the information ratio might decrease over time.
We should mention here that a decreasing information ratio was exploited by \cite{DVX21} in a recent analysis of 
Thompson sampling for finite-armed bandits, but there the gain was less dramatic (a logarithm of the number of arms) and no changes to the algorithm were required.

\section{Algorithm for bandit phase retrieval}\label{sec:alg}

We start by showing that \cref{thm:main} holds if the learner is given an action that is constant-factor optimal. In the next section we explain how such an action can be identified with low regret.
Our algorithm uses the explore-then-commit design principle, which is usually only sufficient for $O(n^{2/3})$ regret. The reason we are able to obtain $O(n^{1/2})$ regret is because of the curvature of the action set, a property that was exploited in a similar way in online learning by \cite{HuLaGySz17} and in partial monitoring by \cite{KLK20}.

\begin{theorem}\label{thm:opt}
Suppose the learner is given an action $\warm \in \ball^d_1$ such that $\ip{\warm, \theta_\star}^2 \geq \alpha r^2$ for some universal constant $\alpha \in (0,1]$.
Then there exists a policy $\pi$ for which the regret is at most $\Reg_n(\pi, \theta_\star) \leq \const \cdot d \sqrt{n \log(n)}$.
\end{theorem}

\begin{proof}
By choosing the sign of $\theta_\star$, assume without loss of generality that $\ip{\warm, \theta_\star} \geq r \sqrt{\alpha}$. Let
\begin{align*}
m = \ceil{4d \sqrt{n \log(n)} / r^2} \qquad \text{and} \qquad \lambda = \min\left(\frac{1}{2},\, \frac{\sqrt{\alpha}}{4}\right) \,.
\end{align*}
If $m \geq n$, then the regret of any policy is upper bounded $n r^2 \leq \const d \sqrt{n \log(n)}$, so for the rest of the proof we assume that $m < n$.
For the first $m$ rounds the policy cycles over the 
$2d$ actions $\{(1 - \lambda) \warm \pm \lambda e_k : k \in [d]\}$. 
The constrained least squares estimator of $\theta_\star$ based on the data collected over $m$ rounds is
\begin{align}\label{eqn:CLS}
\hat \theta = \argmin \{\cL(\theta) : \theta \in \ball^d_r \text{ and } \ip{\warm, \theta} \geq r \sqrt{\alpha}\}\,,
\end{align}
where $\cL(\theta) = \frac{1}{2} \sum_{t=1}^m (X_t - \ip{A_t, \theta}^2)^2$.
For the remaining $n - m$ rounds the algorithm plays $A_t = \widehat A = \hat \theta / \norm{\hat \theta}$. Then,
\begin{align}
\beta 
&\triangleq 9 + d \log(98m) \\
\tag{by \cref{cor:conc}}  &\geq \E\left[\sum_{t=1}^m \ip{A_t, \hat \theta - \theta_\star}^2 \ip{A_t, \hat \theta + \theta_\star}^2\right] \nonumber \\
&\geq \frac{\alpha \norm{\theta_\star}^2}{4} \E\left[\sum_{t=1}^m \ip{A_t, \hat \theta - \theta_\star}^2\right] \label{eq:alg1} \\
&\geq \frac{\alpha \lambda^2 \norm{\theta_\star}^2}{2} \left(\frac{m}{2d}-1\right) \E\left[\norm{\hat \theta - \theta_\star}^2\right]\,.  \label{eq:alg2}
\end{align}
where in \cref{eq:alg1} we used that for some $k \in [d]$,
\begin{align*}
\ip{A_t, \hat \theta + \theta_\star} 
&= \ip{(1 - \lambda)\warm \pm \lambda e_k, \hat \theta + \theta_\star}\geq (1 - \lambda) \ip{\warm, \hat \theta + \theta_\star} - \lambda \norm{\hat \theta + \theta_\star} \\
&\geq 2(1 - \lambda) \sqrt{\alpha} r - 2\lambda r 
\tag{by definition of $\lambda$} \geq \frac{\sqrt{\alpha}}{2}\norm{\theta_\star}\,.
\end{align*}
\cref{eq:alg2} follows because for any $k \in [d]$,
\begin{align*}
\sum_{\sigma \in \pm 1}
\ip{(1 - \lambda) \warm + \sigma \lambda e_k, \hat \theta - \theta_\star}^2 
&= 2(1 - \lambda)^2 \ip{\warm, \hat \theta - \theta_\star}^2 + 2 \lambda^2 \ip{e_k, \hat \theta - \theta_\star}^2 \\
& \geq 2 \lambda^2 \ip{e_k, \hat \theta - \theta_\star}^2\,,
\end{align*}
which implies that
\begin{align*}
\sum_{t=1}^m \ip{A_t, \hat \theta - \theta_\star}^2
&\geq \floor{\frac{m}{2d}} \sum_{k=1}^d \left(2\lambda^2 \ip{e_k, \hat \theta - \theta_\star}^2\right) 
\geq 2\lambda^2 \left(\frac{m}{2d} - 1\right) \norm{\hat \theta - \theta_\star}^2\,. 
\end{align*}
Rearranging \cref{eq:alg2} and using the definition of $\beta$ shows that
\begin{align*}
\E\left[\norm{\hat \theta - \theta}^2\right] \leq \frac{2 \beta}{\alpha \lambda^2 \norm{\theta_\star}^2} \frac{1}{\frac{m}{2d} - 1} \leq \const \frac{d^2 \log(m)}{m \norm{\theta_\star}^2} \,.
\end{align*}
Letting $a_\star = \theta_\star / \norm{\theta_\star}$ be the optimal action,
the regret is bounded by
\begin{align*}
\Reg_n(\pi, \theta_\star) 
&\leq m \norm{\theta_\star}^2 + n \E\left[\ip{a_\star, \theta_\star}^2 - \ip{\widehat A, \theta_\star}^2\right] \\
&= m \norm{\theta_\star}^2 + n \E\left[\ip{a_\star - \widehat A, \theta_\star} \ip{a_\star + \widehat A, \theta_\star}\right] \\
&\leq m \norm{\theta_\star}^2 + 2n \E\left[\left|\ip{a_\star - \widehat A, \theta_\star}\right| \norm{\theta_\star}\right] \\
\tag{by \cref{lem:curve}} &\leq m \norm{\theta_\star}^2 + 4n \E\left[\norm{\theta_\star - \hat \theta}^2\right] \\
&\leq m \norm{\theta_\star}^2 + \const \frac{nd^2 \log(m)}{m \norm{\theta_\star}^2} \leq \const d \sqrt{n \log(n)}\,. \qedhere
\end{align*}
\end{proof}
\section{Finding a constant-factor optimal action}\label{sec:warm}

To establish \cref{thm:main} we show there exists an algorithm that interacts with the bandit for a random number of rounds
and outputs an action $\warm$ that with high probability satisfies $\ip{\warm, \theta_\star}^2 \geq r^2 / 64$. 
Furthermore, the procedure suffers small regret in expectation. 

\begin{theorem}\label{thm:warm}
Let $T$ be the random number of rounds that \cref{alg:warm} interacts with the bandit, which cannot be more than $n$, and let $\warm \in \sphere^{d-1}_1$ be its output.
Then,
\begin{enumerate}
\item $\E[T] \leq \const \frac{d^2}{r^4} \log(n) \log(d)$.
\item With probability at least $1 - 1/n$, either $T = n$ or $\ip{\warm, \theta_\star}^2 \geq \frac{r^2}{64}$.
\end{enumerate}
\end{theorem}

What is interesting about \cref{alg:warm} is that it uses what it has learned in early iterations to increase the
statistical efficiency of its estimation.

\lstset{emph={
    for, loop, until, do, to, if, else, input, return},
    emphstyle={\bfseries},
    basicstyle=\linespread{1.8},
    commentstyle=\color{green}\ttfamily,
    escapechar=£,
    numbers=left,
    xleftmargin=5.0ex,
}%

\begin{algorithm}
\begin{lstlisting}[mathescape=true]
$\tau = r / \sqrt{d}$ and $m = \ceil{\frac{8}{\tau^4} \log\left(2n^2\right)}$ £\label{line:m1}£
do
  sample $v$ uniformly from $\sphere^{d-1}_1$
  play $A_t = v$ for $m$ rounds and compute average reward $\bar X$
loop until $\bar X \geq \tau^2$  £\label{line:stop1}£
$\cE = \{v \sqrt{\bar X}\}$
for $k = 2$ to $d$
  $\displaystyle \beta = 9(\log(98) + 4 \log(n))$ and $\displaystyle m = \ceil{\frac{64 d^2 \beta}{kr^4}}$ and $\displaystyle u = \frac{\sum_{w \in \cE} w}{\norm{\sum_{w \in \cE} w}}$ £\label{line:u}£
  do
    sample $v$ uniformly from $\laspan(\cE)^\perp \cap \sphere^{d-1}_1$ £\label{line:v}£
    play $A_t \in \{(u+v)/\sqrt{2}, (u-v)/\sqrt{2}, u\}$ for $3m$ rounds 
    find least squares estimator $\hat \theta$  constrained to $\ball^d_1$ 
  loop until $\ip{v, \hat \theta}^2 \geq \tau^2$ 
  $\cE = \cE \cup \{ \ip{v, \hat \theta} v\}$
loop until $\sum_{w \in \cE} \norm{w}^2 \geq r^2/16$ £\label{line:stop}£
return $\warm = \frac{\sum_{w \in \cE} w}{\norm{\sum_{w \in \cE} w}}$
\end{lstlisting}
\caption{
The procedure operates in $d$ iterations. The first iteration is implemented in Lines 1--5 and the remaining $d-1$ iterations in Lines 7--15. 
}
\label{alg:warm}
\end{algorithm}

\begin{proof}
Note that the vectors $u$ and $v$ computed in each iteration are orthogonal, which means that $\norm{(u+v)/\sqrt{2}} = \norm{(u-v)/\sqrt{2}} = \norm{v} = 1$. 
Hence the actions of the algorithm are always in $\ball^d_1$.
The main argument of the proof is based on an induction to show that with high probability when the execution of the algorithm ends, there exists an $s \in \{\pm 1\}$ 
such that for all $w \in \cE$,
\begin{enumerate}
\item[(a)] $w \in \cE$, $\norm{w}^2 \geq \tau^2$; and
\item[(b)] $s \ip{w, \theta_\star} \in [\frac{1}{2}\norm{w}^2, 2 \norm{w}^2]$.
\end{enumerate}
We proceed in five steps. First, we prove that if the above holds and the algorithm halts before $n$ rounds are over, then the vector returned
is a suitable approximation of $\theta_\star / \norm{\theta_\star}$. Second, we upper bound the probability of certain bad events.
In the third and fourth steps we prove the base case and induction step for (a) and (b). In the last step we bound the expected running time.

\paragraph{Step 1: Correctness}
Suppose that (a) and (b) above hold and the algorithm halts at the end of iteration $k$. Then,
\begin{align*}
\ip{\warm, \theta_\star}^2 
=\bip{\frac{\sum_{w \in \cE} w}{\norm{\sum_{w \in \cE} w}},\, \theta_\star}^2
\geq \frac{1}{4} \sum_{w \in \cE} \norm{w}^2
\geq \frac{r^2}{64}\,.
\end{align*}
where the first inequality follows from orthogonality of $w \in \cE$ and (b) above.
The second inequality follows from the stopping condition in Line~\ref{line:stop} of \cref{alg:warm}, (a) above and the definition of $\tau$. Part (2) of the theorem follows by showing that (a) and (b) above hold with probability at least $1 - 1/n$. 

\paragraph{Step 2: Failure events}
The algorithm computes some kind of estimator at the end of each do/loop.
Since the algorithm cannot play more than $n$ actions, the number of estimators computed is naively upper bounded by $n$.
A union bound over all estimates and the concentration bounds in \cref{lem:gaussian} and \cref{cor:conc}
show that with probability at least $1 - 1/n$ the following both hold:
\begin{itemize}[leftmargin=*]
\item For all $v$ sampled in the first iteration and corresponding average rewards $\bar X$,
\begin{align}
\left|\bar X - \ip{v, \theta_\star}^2\right| \leq \sqrt{\frac{2 \log(2n^2)}{m}} = \frac{\tau^2}{2}\,,
\label{eq:warm0}
\end{align}
where $m$ is defined in Line~\ref{line:m1} of \cref{alg:warm}.
\item Let $\cD = (A_s)_s$ be the actions played in the inner loop of some iteration $k \geq 2$ and $\hat \theta$ be the corresponding least-squares estimator. 
Then,
\begin{align}
\sum_{a \in \cD} \ip{a, \hat \theta - \theta_\star}^2 \ip{a, \hat \theta + \theta_\star}^2 \leq 9(\log(98) + 4 \log(n)) \triangleq \beta\,.
\label{eq:warm1}
\end{align}
\end{itemize}
We assume both of the above hold in all relevant iterations for the remainder.

\paragraph{Step 3: Base case}
The next step is to show that (a) and (b) hold with high probability after the first iteration.
Consider the operation of the algorithm in the inner loop. After sampling $v \in \sphere^{d-1}_1$, the algorithm plays $v$ for $m$ rounds and computes
the average reward.
Let $v$ be the last sampled action before the iteration halts and $w = v \sqrt{\bar X}$.
By the stopping condition in Line~\ref{line:stop1}, $\norm{w}^2 = \bar X \geq \tau^2$.
Without loss of generality, we choose the sign of $\theta_\star$ so that $\ip{v, \theta_\star} \geq 0$. Then by \cref{eq:warm0},
\begin{align*}
\ip{w, \theta_\star}
= \sqrt{\bar X} \ip{v, \theta_\star}
\in \left[\frac{1}{2} \norm{w}^2, 2 \norm{w}^2\right]\,. 
\end{align*}
This establishes the base case.

\paragraph{Step 4: Inductive step}
Assume that (a) and (b) above hold for $\cE$ at the end of iteration $k$. Let $u$ be the value computed in Line~\ref{line:u} of \cref{alg:warm}. Then,
\begin{align*}
\ip{u, \theta_\star}
&= \frac{\sum_{w \in \cE} \ip{w, \theta_\star}}{\sqrt{\sum_{w \in \cE} \norm{w}^2}} 
\geq \frac{1}{2} \sqrt{\sum_{w \in \cE} \norm{w}^2} 
\geq \frac{\tau \sqrt{k}}{2}\,.
\end{align*}
Let $\cA = \{(u + v)/\sqrt{2}, (u - v)/\sqrt{2}, v\}$, which is the set of actions played in the inner loop of iteration $k+1$ after sampling $v$ for the last time. 
Let $\hat \theta$ be the corresponding least-squares estimate. We consider two cases. First, if $\ip{u, \hat \theta + \theta_\star} \geq 2 |\ip{v, \hat \theta + \theta_\star}|$, then by \cref{eq:warm1},
\begin{align*}
\frac{\beta}{m}
&\geq \sum_{a \in \cA} \ip{a, \theta_\star - \hat \theta}^2 \ip{a, \theta_\star + \hat \theta}^2 \\ 
&\geq \frac{1}{16} \ip{u + v, \theta_\star - \hat \theta}^2 \ip{u, \theta_\star}^2 + \frac{1}{16} \ip{u - v, \theta_\star - \hat \theta}^2 \ip{u, \theta_\star}^2 \\
&\geq \frac{1}{8} \ip{v, \theta_\star - \hat \theta}^2 \ip{u, \theta_\star}^2 \geq \frac{k r^2}{16d} \ip{v, \theta_\star - \hat \theta}^2\,. 
\end{align*}
Rearranging shows that
\begin{align}
\ip{v, \theta_\star - \hat \theta}^2 \leq \frac{16d \beta}{m k r^2} \leq \frac{\tau^2}{4}\,. 
\label{eq:warm}
\end{align}
For the second case, $\ip{u, \hat \theta + \theta_\star} \leq 2 |\ip{v, \hat \theta + \theta_\star}|$. Then,
\begin{align*}
\frac{\beta}{m}
\geq \sum_{a \in \cA} \ip{a, \theta_\star - \hat \theta}^2 \ip{a, \theta_\star + \hat \theta}^2 
\geq \frac{1}{4} \ip{v, \hat \theta - \hat \theta}^2 \ip{u, \hat \theta + \theta_\star} 
\geq \frac{k r^2}{8d} \ip{v, \theta_\star - \hat \theta}^2\,.
\end{align*}
And again, \cref{eq:warm} holds.
Summarising, $\ip{v, \hat \theta}$ is an estimator of $\ip{v, \theta_\star}$ up to accuracy $\tau / 2$.
By the definition of the algorithm, the iteration only ends if $|\ip{v, \hat \theta}| \geq \tau$.
Therefore, with $w = \ip{v, \hat \theta} v$, we have
$\norm{w}^2 = \ip{v, \hat \theta}^2 \geq \tau^2$.
Furthermore, 
$
\ip{w, \theta_\star} = \ip{v, \hat \theta} \ip{v, \theta_\star} 
\in [\norm{w}^2/2, 2 \norm{w}^2]\,.
$
Therefore if (a) and (b) hold for $\cE$ computed after iteration $k$, they also hold for $\cE$ computed after iteration $k+1$.

\paragraph{Step 5: Running time}
The length of an iteration is determined by the corresponding value of $m$ and the number of samples of $v$. The former is an iteration-dependent constant, while the latter depends principally on how many samples are needed before $|\ip{v, \theta_\star}|$ is suitably large. The law of $\nu$ 
is the uniform distribution on $\sphere^{d-1}_1 \cap \laspan(\cE)^\perp$, which is the uniform distribution on a sphere of dimension $d-1-|\cE|$ embedded in $\R^d$. 
The squared norm of the projection of $\theta_\star$ onto $\laspan(\cE)^\perp$ is
\begin{align*}
\norm{\theta_\star}^2 - \sum_{w \in \cE} \frac{\ip{\theta_\star, w}^2}{\norm{w}^2}
\geq r^2 - 4\sum_{w \in \cE} \norm{w}^2 
\geq \frac{r^2}{2}\,,
\end{align*}
where we used (a) of the induction and the stopping condition in Line~\ref{line:stop}.
Therefore, when $v$ is sampled uniformly from $\sphere^{d-1}_1 \cap \laspan(\cE)^\perp$, by \cref{lem:sphere},
\begin{align*}
\bbP\left(\ip{v, \theta_\star}^2 \geq 3\tau^2/2\right)
&= \bbP\left(\ip{v, \theta_\star}^2 \geq \frac{3r^2}{2d}\right) 
\geq \const > 0\,,
\end{align*}
Furthermore, by the concentration analysis in the previous step, an iteration will end once a $v$ has been sampled
for which $\ip{v, \theta_\star}^2 \geq 3\tau^2 / 2$. 
Hence, the expected number of times the algorithm samples $v$ per iteration is constant and a simple calculation using 
the definition of $m$ in Lines~\ref{line:m1} and \ref{line:u} shows that the expected number of rounds used by the
algorithm is at most 
\begin{align*}
\E[T] &\leq \const \frac{d^2}{r^4} \log\left(n\right) \log(d)\,.
\qedhere
\end{align*}
\end{proof}

\section{Proof of \cref{thm:main} and \cref{thm:simple-upper}}

\begin{proof}[Proof of \cref{thm:main}]
Run \cref{alg:warm} and if it halts, feed the returned action to the input of the explore-then-commit algorithm analysed in \cref{thm:opt}.
\cref{alg:warm} fails to return suitable action with probability at most $1/n$, so the contribution of this event to the regret
is negligible. By \cref{thm:warm}, the regret incurred by \cref{alg:warm} is bounded by
\begin{align*}
\E\left[\sum_{t=1}^T \left(r^2 - \ip{A_t, \theta_\star}^2\right)\right]
&\leq r^2 \E[T] \\
&\leq \min\left(n r^2, \const \frac{d^2}{r^4} \log(d) \log(n)\right) \\
&\leq \const d \sqrt{n \log(d) \log(n)}\,.
\end{align*}
Combining this with the regret bound established in \cref{thm:opt} yields the result.
\end{proof}

\begin{proof}[Proof of \cref{thm:simple-upper}]
We use a standard reduction \cite[Chapter 33]{LS20book}.
Let $\pi$ be the policy used in the proof of \cref{thm:main} with $\widehat A_\star$ sampled uniformly from $(A_t)_{t=1}^n$. By \cref{thm:main},
\begin{align*}
    \reg_n(\pi, \theta_\star) 
    &= \frac{1}{n} \left(n r^2 -  \E\left[ \sum_{t=1}^n \ip{A_t, \theta_\star}^2\right]\right) 
    = \frac{\Reg_n(\pi, \theta_\star)}{n} 
    \leq \const d \sqrt{\frac{\log(n) \log(d)}{n}}\,.
\end{align*}
\end{proof}

\input{lower_bound}

\section{Discussion}\label{sec:disc}

\paragraph{Unknown radius}
The assumption that $r = \norm{\theta_\star}$ is known to the learner is easily relaxed by estimating $\norm{\theta_\star}$.
Note first that all our analysis holds with only trivial modifications if $r \in [\frac{1}{2} \norm{\theta_\star}, \norm{\theta_\star}]$.
Next, if $A$ is sampled uniformly from $\sphere^{d-1}_1$ and $X = \ip{A, \theta_\star}^2 + \eta$ and $\eta$ is a standard Gaussian, then
$\E[X] = \frac{1}{d} \norm{\theta_\star}^2$ and $\Var[X] = 1 + 2(d-1)/(d^3 + 2d^2)= \Theta(1)$.
Therefore $\norm{\theta_\star}$ can be estimated to within an arbitrary multiplicative factor and at confidence level $1 - 1/n$ 
using $\const d^2 / \norm{\theta_\star}^4 \log(n)$ interactions with the bandit. 

\paragraph{Computation complexity} The only computational challenge is
finding the least squares estimates, which is a non-convex optimisation problem. \cite{candes2015phase} proposed a Wirtinger flow algorithm that starts with a spectral initialization, and then refines this initial estimate using a local update like gradient descent. The computational complexity of the Wirtinger flow algorithm with $\varepsilon$-accuracy is $O(nd^2\log(1/\varepsilon))$ where $n$ is the number of samples.

\paragraph{Adversarial setting}
\cite{KN19} study the adversarial version of this problem, where the learner observes $\ip{A_t, \theta_t}^2$ and
$(\theta_t)_{t=1}^n$ is an adversarially sequence with $\theta_t \in \ball_1^d$ for all $t$.
They prove an upper bound of $\Reg_n = O(d \sqrt{n \log(n)})$ and a lower bound of $\Omega(\sqrt{dn})$.
Natural attempts at improving the lower bound all fail. We believe that the upper bound is loose, but proving
this remains delicate. No warm starting procedure will work anymore because the information gained may be useless
in the presence of a change point. 
New ideas are needed.

\paragraph{Rank-$\bm p$}
Perhaps the most natural open question is whether or not our analysis can be extended to the low rank bandit problem without our particular assumptions on the action set and environments matrices. 

\paragraph{Principled algorithms} 
Can optimism or information-directed sampling be made to work? The main challenge is to understand the sample paths of these algoritms \textit{before} learning takes place. 

\bibliographystyle{plainnat}
\bibliography{all}

\appendix

\section{Technical lemma}

\begin{lemma}[\citealt{KLK20}]\label{lem:curve}
$\bip{\frac{\theta}{\norm{\theta}} - \frac{\varphi}{\norm{\varphi}},\, \theta} \leq \frac{2}{\norm{\theta}} \norm{\theta - \varphi}^2$.
\end{lemma}

\begin{lemma}[\citealt{BLM13}]\label{lem:gaussian}
Let $(X_t)_{t=1}^n$ be independent standard Gaussian random variables and $(a_t)_{t=1}^n$ be constants. Then,
\begin{align*}
\bbP\left(\left|\frac{1}{n} \sum_{t=1}^n a_t X_t\right| \geq \sqrt{\frac{2 \sum_{t=1}^n a_t^2 \log(2/\delta)}{n}}\right) \leq \delta\,.
\end{align*}
\end{lemma}

\begin{lemma}\label{lem:sphere}
Let $V \subset \R^d$ be a $m$-dimensional subspace and
let $X$ be sampled uniformly from $\sphere^{d-1}_1 \cap V$. Then for all $\varphi \in V$,
\begin{align*}
\bbP\left(\ip{X, \varphi}^2 \geq \frac{\norm{\varphi}^2}{m}\right)
\geq \const > 0\,.
\end{align*}
\end{lemma}

\begin{proof}
Use the fact that if $Z \in \R^m$ is a standard Gaussian, then
\begin{align*}
    \ip{X, \varphi}^2 \stackrel{d}= \frac{Z_1 \norm{\varphi}^2}{\norm{Z}}\,.
\end{align*}
Then use standard concentration for the Gaussian and $\chi$-squared distributions and naive union bounding \citep{LM00}.
Alternatively, use the explicit form for the distribution of $X$ in combination with elementary bounds on the regularised incomplete beta function.
\end{proof}

\section{Ordinary least squares}\label{sec:conc}
Here we provide some routine results for least-squares estimation of $\theta_\star$.
Suppose that $(A_t)_{t=1}^n$ are fixed and $(\eta_t)_{t=1}^n$ are independent $1$-subgaussian
random variables and $X_t = \ip{A_t, \theta_\star}^2 + \eta_t$. 
The least-squares estimator of $\theta_\star$ constrained to $\Theta \subset \ball^d_1$ is
\begin{align*}
\hat \theta = \argmax_{\theta \in \ball^d} \cL(\theta) \qquad \text{with} \qquad \cL(\theta) = \argmax_{\theta \in \Theta} \frac{1}{2}  \sum_{t=1}^n \left(X_t - \ip{A_t, \theta}^2\right)^2\,.
\end{align*}
The symmetry of the problem means that $\cL(\theta) = \cL(-\theta)$ for all $\theta \in \R^d$, which means there is no hope
that $\hat \theta$ might be close to $\theta_\star$.
What is true is that for suitably exploratory $(A_t)$, $\hat \theta$ is close to either $\theta_\star$ or $-\theta_\star$.

\begin{theorem}\label{thm:conc}
Suppose that $\theta_\star \in \ball^d$ and
$\hat \theta = \argmin_{\theta \in \Theta} \cL(\theta)$. Then, for any $\delta \in (0,1)$, with probability at least $1 - \delta$,
\begin{align*}
\bbP\left(\sum_{t=1}^n \ip{\hat \theta - \theta_\star, A_t}^2 \ip{\hat \theta + \theta_\star, A_t}^2 
\geq 9 \log \left(\frac{N_{1/(32n)}(\Theta)}{\delta}\right)\right) \leq \delta\,,
\end{align*}
where $N_\epsilon(\Theta) = \min\{|\cC| : \cC \subset \R^d, \forall x \in \Theta, \min_{y \in \cC} \norm{x - y} \leq \epsilon\}$.
\end{theorem}

\begin{proof}
Since $\theta_\star \in \ball^d$ by assumption, it follows that
\begin{align*}
0 
&\leq \cL(\theta_\star) - \cL(\hat \theta) \\
&= -\frac{1}{2} \sum_{t=1}^n \ip{A_t, \hat \theta - \theta_\star}^2 \ip{A_t, \hat \theta + \theta_\star}^2
+ \sum_{t=1}^n \eta_t \ip{A_t, \hat \theta - \theta_\star} \ip{A_t, \hat \theta + \theta_\star} \,.
\end{align*}
Let $\epsilon = 1/(32n)$ and $\cC \subset \R^d$ be such that for all $x \in \Theta$ there exists a $y \in \cC$ such that $\norm{x - y} \leq \epsilon$
and $|\cC| = N_\epsilon(\Theta)$.
Since $A_t$ are fixed, by a union bound and standard Gaussian tail bounds, with probability at least $1 - |\cC| \delta$,
\begin{align*}
    \left|\sum_{t=1}^n \eta_t \ip{A_t, \alpha - \theta_\star}\ip{A_t, \alpha + \theta_\star}\right| 
    &\leq \sqrt{2 \sum_{t=1}^n \ip{A_t, \alpha - \theta_\star}^2 \ip{A_t, \alpha + \theta_\star}^2 \log\left(\frac{1}{\delta}\right)}\,. 
\end{align*}
On this event and letting $\alpha \in \cC$ be such that $\norm{\alpha - \hat \theta} \leq \epsilon$. Then, with $\Delta = \alpha - \hat \theta$,
\begin{align*}
&\sum_{t=1}^n \ip{A_t, \hat \theta - \theta_\star}^2 \ip{A_t, \hat \theta + \theta_\star}^2
\leq \sqrt{8 \sum_{t=1}^n \ip{A_t, \alpha - \theta_\star}^2 \ip{A_t, \alpha + \theta_\star}^2 \log\left(\frac{1}{\delta}\right)} \\
&= \sqrt{8 \sum_{t=1}^n \ip{A_t, \Delta + \hat \theta - \theta_\star}^2 \ip{A_t, \Delta + \hat \theta + \theta_\star}^2 \log\left(\frac{1}{\delta}\right)} \\
&\leq \sqrt{8 \sum_{t=1}^n 
\left(\ip{A_t, \hat \theta - \theta_\star}^2 + 2 \epsilon + \epsilon^2\right)
\left(\ip{A_t, \hat \theta + \theta_\star}^2 + 4 \epsilon + \epsilon^2\right) \log\left(\frac{1}{\delta}\right)} \\
&\leq \sqrt{8 \sum_{t=1}^n \left(\ip{A_t, \hat \theta - \theta_\star}^2 \ip{A_t, \hat \theta + \theta_\star}^2 + 12\epsilon + 13\epsilon^2 + 6\epsilon^3 + \epsilon^4\right) \log\left(\frac{1}{\delta}\right)} \\
&\leq \sqrt{8 \sum_{t=1}^n \left(\ip{A_t, \hat \theta - \theta_\star}^2 \ip{A_t, \hat \theta + \theta_\star}^2 + 32\epsilon\right) \log\left(\frac{1}{\delta}\right)} \\
&\leq \sqrt{\left(1 + 8 \sum_{t=1}^n \ip{A_t, \hat \theta - \theta_\star}^2 \ip{A_t, \hat \theta + \theta_\star}^2 \right) \log\left(\frac{1}{\delta}\right)} \,, 
\end{align*}
where in the final inequality we chose $\epsilon = 1/(32n)$.
Solving for the left-hand side and naive simplification shows that
\begin{align*}
\sum_{t=1}^n \ip{A_t, \hat \theta - \theta_\star}^2 \leq 9 \log\left(\frac{1}{\delta}\right)\,.
\end{align*}
To summarise we have shown that with probability at least $1 - \delta$,
\begin{align*}
\sum_{t=1}^n \ip{A_t, \hat \theta - \theta_\star}^2 \ip{A_t, \hat \theta + \theta_\star}^2 
&\leq 9 \log\left(\frac{|\cC|}{\delta}\right)
= 9 \log\left(\frac{N_\epsilon(\Theta)}{\delta}\right) \,.
\qedhere
\end{align*}
\end{proof}

Standard results show that when $\Theta \subset \ball^d_1$ has dimension $k$, then $\log N_\epsilon(\Theta) \leq m \log(3/\epsilon)$.
From this one obtains the following corollary:

\begin{corollary}\label{cor:conc}
Under the same conditions as \cref{thm:conc} and when $\Theta \subset \ball^d_1$ and $\dim(\laspan(A_1,\ldots,A_n)) = k$:
\begin{enumerate}
\item[(a)] $\displaystyle \bbP\left(\sum_{t=1}^n \ip{A_t, \hat \theta - \theta_\star}^2 \ip{A_t, \hat \theta + \theta_\star}^2\right] \geq 9 \left(\log(1/\delta) + k \log(98n)\right) \leq \delta$.
\item[(b)] $\displaystyle \E\left[\sum_{t=1}^n \ip{A_t, \hat \theta - \theta_\star}^2 \ip{A_t, \hat \theta + \theta_\star}^2\right] \leq 9 \left(1 + k \log(98n)\right)$.
\end{enumerate}
\end{corollary}

\section{Proof of \cref{thm:simple-lower}}\label{app:lower}
Let $\pi$ be a fixed policy and for $\theta \in \R^d$ let $\mathbb P_{\theta}$ be the measure on the sequence of outcomes $H_n = (A_1,X_1,\ldots,A_n,X_n)$ induced by the interaction between $\pi$ and the phase retrieval model determined by $\theta$. Let $\E_\theta$ denote the expectation with respect to $\bbP_\theta$. Let $r$ be a positive constant to be tuned subsequently and $\sigma$ be the uniform (Haar) measure on
$\sphere^{d-1}_r$. Let $\mathbb Q = \int \mathbb P_{\theta} \d \sigma(\theta)$ be the Bayesian mixture measure. For $\theta \in \R^d$, let $\cE_\theta$ be the event given by
\begin{equation*}
    \cE_{\theta} = \left\{\langle \widehat{A}_{\star}, \theta \rangle^2\geq \frac{3}{4}r^2\right\} \,. 
\end{equation*}
By Fano's inequality \citep[Lemma 5]{gerchinovitz2020fano},
\begin{align}\label{eqn:Fano}
    \int_{\sphere^{d-1}_r} \mathbb P_{\theta}(\cE_{\theta})\d{\sigma}(\theta) \leq \frac{\log 2 + \int_{\sphere^{d-1}_r} \KL(\mathbb P_{\theta}, \mathbb Q) \d{\sigma}(\theta)}{-\log\left(\int_{\sphere^{d-1}_r} \mathbb Q(\cE_{\theta})\d{\sigma}(\theta)\right)\,}.
\end{align}
We now bound the numerator and denominator in \cref{eqn:Fano} to show that the right-hand side is at most $1/2$ and then complete the proof using the definition of the regret and $\cE_\theta$.

\paragraph{Step 1: Bounding the denominator in \cref{eqn:Fano} }
By exchanging the order of integrals in the denominator of \cref{eqn:Fano}, it follows that
\begin{align}
-\log\left(\int_{\sphere^{d-1}_r} \bbQ(\cE_\theta) \d{\sigma}(\theta)\right) 
= -\log\left(\int \int_{\sphere^{d-1}_r} \sind_{\cE_\theta} \d{\sigma}(\theta) \d{\bbQ}\right)\,.
\label{eq:denom}
\end{align}
If $U$ is sampled uniformly from $\sphere^{d-1}_1$, then by a concentration bound for spherical measures  \citep[Lemma 2.2]{dasgupta2003elementary},
\begin{equation*}
    \mathbb P\left(U_1^2\geq\delta /d\right)\leq \exp(-\delta/4) \text{ for all } \delta > 6\,.
\end{equation*}
By scaling and rotating and choosing $\delta= \frac{3}{4}d$, it follows that for any $\widehat A_\star \in \ball^d_1$,
\begin{align*}
    \int_{\sphere^{d-1}_r} \sind\left(\ip{\widehat A_\star, \theta}^2 \geq \frac{3r^2}{4}\right) \d{\sigma}(\theta) \leq \exp\left(-3d/16\right)\,.
\end{align*}
Therefore, by \cref{eq:denom},
\begin{equation*}
    -\log\left(\int_{\sphere^{d-1}_r} \bbQ(\cE_\theta) \d{\sigma}(\theta)\right) \geq \frac{3d}{16}\,.
\end{equation*}
\paragraph{Step 2: Bounding the numerator in \cref{eqn:Fano}}
By the convexity of KL-divergence,
\begin{equation*}
\begin{split}
      \int_{\sphere^{d-1}_r} \KL(\mathbb P_{\theta}, \mathbb Q)  \d\sigma(\theta) &= \int_{\sphere^{d-1}_r} \KL\left(\mathbb P_{\theta}, \int_{\sphere^{d-1}_r}\mathbb P_{\alpha} \d\sigma(\alpha)\right) \d\sigma(\theta)\\
      &\leq \int_{\sphere^{d-1}_r} \int_{\sphere^{d-1}_r}\KL\left(\mathbb P_{\theta}, \mathbb P_{\alpha}\right) \d\sigma(\alpha)  \d\sigma(\theta)\, .
\end{split}
\end{equation*}
By the chain rule of KL-divergence,
\begin{equation*}
    \KL\left(\mathbb P_{\theta}, \mathbb P_{\alpha}\right) = \E_\theta\left[\sum_{t=1}^n\KL\left(\mathbb P_{\theta}(Y_t=\cdot|A_t), \mathbb P_{\alpha}(Y_t=\cdot| A_t)\right)\right]\, .
\end{equation*}
A straightforward computation leads to
\begin{equation*}
\begin{split}
    \KL\left(\mathbb P_{\theta}, \mathbb P_{\alpha}\right) &= \mathbb E_\theta\left[\sum_{t=1}^n\frac{1}{2}\left((A_t^{\top}\theta)^2-(A_t^{\top}\alpha)^2\right)^2\right] \\
    &= \frac{1}{2} \mathbb E_\theta\left[\sum_{t=1}^n\left((A_t^{\top}\theta)^4-2(A_t^{\top}\theta)^2(A_t^{\top}\alpha)^2+(A_t^{\top}\alpha)^4\right)\right].
    \end{split}
\end{equation*}
Since $(A_t)_{t=1}^n$ are independent of $(X_t)_{t=1}^n$, we can interchange the expectation and integral such that 
\begin{align*}
        &\int_{\sphere^{d-1}_r} \KL(\mathbb P_{\theta}, \mathbb Q)  \d\sigma(\theta) \\ 
        &\qquad\leq \sum_{t=1}^n\mathbb E\left[\int_{\theta}\int_{\alpha}\left((A_t^{\top}\theta)^4-2(A_t^{\top}\theta)^2(A_t^{\top}\alpha)^2+(A_t^{\top}\alpha)^4\right)\d \sigma(\alpha)\d\nu(\theta)\right]\,,
\end{align*}
where the expectation is with respect to $(A_t)_{t=1}^n$, which does not depend on $\theta$ by assumption.
When $\theta$ is uniformly on
$\sphere^{d-1}_r$ and $A \in \ball_1^d$ is arbitrary,
\begin{equation*}
    \int_{\sphere^{d-1}_r} \ip{A_t,\theta}^4 \d{\sigma}(\theta) = \frac{3r^4}{d^2+2d} \qquad \text{and} \qquad \int_{\sphere^{d-1}_r} \ip{A_t,\theta}^2 \d{\sigma}(\theta) = \frac{1}{d^2}\, ,
\end{equation*}
where the expectation is taken with respect to $\theta$.
Therefore,
\begin{equation}\label{eqn:KL_bound}
\begin{split}
    \int_{\sphere^{d-1}_\theta} \KL(\mathbb P_{\theta}, \mathbb Q)  \d\sigma(\theta)\leq \frac{3nr^4}{d^2}\, .
     \end{split}
\end{equation}

\paragraph{Step 3: Lower bounding the regret} 
Let $r^2 = \sqrt{d^3 / (32n)}$ 
Combining the previous two steps shows that
\begin{align*}
   \int_{\sphere^{d-1}_r} \mathbb P_\theta(\cE_\theta) \d{\sigma}(\theta)
    \leq \frac{16nr^2}{d^3}
    \leq \frac{1}{2}\,.
\end{align*}
Therefore there exists a $\theta \in \sphere^{d-1}_r$ with $\bbP_\theta(\cE_\theta) \leq 1/2$,
which implies that
\begin{align*}
    \reg_n(\pi, \theta) 
    &= r^2 - \E_\theta\left[\ip{\widehat{A}_\star, \theta}^2\right] 
    \geq \frac{ r^2}{8} 
    \geq \const \frac{d^{3/2}}{\sqrt{n}}\,.
\qedhere
\end{align*}

\end{document}

%% file: lower_bound.tex
\section{Lower bounds}\label{sec:lower}

The proof of \cref{thm:lower} is based on a direct analysis of the Bayesian regret via a change of measure argument. 
Compared to the usual arguments based on Pinsker's inequality, the idea used here does not rely on any boundedness assumptions.
What is important to show is that when $\theta_\star$ is sampled uniformly from $\sphere^{d-1}_r$, then for suitably small $r$ the learner cannot learn much about $\theta_\star$. The proof of \cref{thm:simple-lower} makes use of Fano's inequality and can be found in \cref{app:lower}.
The technique used in the proof below also yields a lower bound that matches the upper bound in \cref{thm:simple-upper} up to logarithmic factors.

\begin{proof}[Proof of Theorem~\ref{thm:lower}]
Let $r$ be a positive constant to be tuned subsequently and $\sigma$ be the uniform (Haar) measure on
$\sphere^{d-1}_r$. Let $\pi$ be an arbitrary policy.
Given $\theta \in \R^d$, let $\bbP_\theta$ be the measure on interaction sequences induced by the interaction between $\pi$ and the bandit determined by $\theta$. 
Let $\bbP = \sigma \otimes \bbP_\theta$ be the product of $\sigma$ and probability kernel $(\bbP_\theta : \theta \in \sphere^{d-1}_r)$.
As usual, let $\E_\theta$ denote the expectation with respect to $\bbP_\theta$ and $\E$ the expectation with respect to $\bbP$.
Let $Z_t = \sum_{s=1}^t X_s \ip{A_s, \theta}^2 - \frac{1}{2} \ip{A_s, \theta}^4$.
Then,
\begin{align*}
&\E[\ip{A_t, \theta}^4 \sind(Z_{t-1} \leq \gamma)]
= \int_{\sphere_r^{d-1}} \E_\theta[\ip{A_t, \theta}^4 \sind(Z_{t-1} \leq \gamma)] \d{\sigma}(\theta) \\
&= \int_{\sphere_r^{d-1}} \E_{\zeros}\left[\ip{A_t, \theta}^4 \sind(Z_{t-1} \leq \gamma) \exp\left(-\frac{1}{2} \sum_{s=1}^{t-1}  \left((X_s - \ip{A_s, \theta}^2)^2 - X_s^2\right)\right)\right] \d{\sigma}(\theta)\\
&= \int_{\sphere_r^{d-1}} \E_{\zeros} \left[\ip{A_t, \theta}^4 \sind(Z_{t-1} \leq \gamma) \exp(Z_{t-1})\right] \d{\sigma}(\theta) \\
&\leq \exp(\gamma) \int_{\sphere_r^{d-1}} \E_{\zeros}\left[\ip{A_t, \theta}^4\right] \d{\sigma}(\theta) 
= \frac{3\exp(\gamma) r^4}{d^2 + 2d}\,.
\end{align*}
Next, define a random time
\begin{align*}
    \tau = \min\left\{t : Z_t \geq \gamma \text{ or } \sum_{s=1}^t \ip{A_s, \theta}^4 \geq 1\right\}
\quad \text{with } \gamma = 1 + \log(4)\,.
\end{align*}
Now,
\begin{align*}
\{\tau \leq n\}
\subset \left\{\sum_{t=1}^\tau X_s \ip{A_s, \theta}^2 - \frac{3}{2} \ip{A_s, \theta}^4 \geq \gamma - 1 \right\}
\cup \left\{\sum_{t=1}^n \ip{A_s, \theta}^4 \sind(\tau \geq t) \geq 1\right\}\,.
\end{align*}
By a union bound,
\begin{align*}
\bbP(\tau \leq n) 
&\leq \underbracket{\bbP\left(\sum_{t=1}^\tau X_s \ip{A_s, \theta}^2 - \frac{3}{2} \ip{A_s, \theta}^4 \geq \gamma -1\right)}_{\textrm{(A)}} \\
&\qquad+ \underbracket{\bbP\left(\sum_{t=1}^n \ip{A_s, \theta}^4 \sind(\tau \geq t) \geq 1\right)}_{\textrm{(B)}}\,.
\end{align*}
By the definition of $X_s$, Markov's inequality and the moment generating function of a Gaussian,
\begin{align*}
\textrm{(A)} 
&\leq \exp(1-\gamma) \E\left[\exp\left(\sum_{t=1}^\tau X_s \ip{A_s, \theta}^2 - \frac{3}{2} \ip{A_s, \theta}^4\right)\right]
\leq \exp(1-\gamma) = \frac{1}{4}\,.
\end{align*}
Meanwhile, by Markov's inequality, 
\begin{align*}
\textrm{(B)} \leq \E\left[\sum_{t=1}^n \ip{A_t, \theta}^4 \sind(\tau \geq t)\right]
\leq \frac{6 \exp(\gamma) n r^4}{d^2 + 2d} = \frac{24 \exp(1) n r^4}{d^2 + 2d} = \frac{1}{4}\,,
\end{align*}
where the last equality follows by choosing $r^2 = \sqrt{(d^2 + 2d)/(96 \exp(1)n)}.$
Hence, $\int_{\sphere^{d-1}_r} \bbP_\theta(\tau \leq n) \d{\sigma}(\theta) = \bbP(\tau \leq n) \leq 1/2$ and so
there exists a $\theta \in \sphere^{d-1}_r$ such that
\begin{align*}
\bbP_\theta\left(\sum_{t=1}^n \ip{A_t, \theta}^4 \leq 1\right) \geq 1/2\,.
\end{align*}
Suppose that $\sum_{t=1}^n \ip{A_t, \theta}^4 \leq 1$, then
by Cauchy-Schwarz, 
\begin{align*}
\sum_{t=1}^n \ip{A_t, \theta}^2 
\leq \sqrt{n \sum_{t=1}^n \ip{A_t, \theta}^4}
\leq \sqrt{n} \,.
\end{align*}
Therefore, $\bbP_\theta(\sum_{t=1}^n \ip{A_t, \theta}^2 \leq \sqrt{n}) \geq 1/2$ and
\begin{align*}
\Reg_n(\pi, \theta) 
&=\E_\theta\left[\sum_{t=1}^n \left(r^2 - \ip{A_t, \theta}^2\right)\right] \\
&\geq \bbP_\theta\left(\sum_{t=1}^n \ip{A_t, \theta}^2 \leq \sqrt{n}\right) (n r^2 - \sqrt{n}) \\
&\geq \frac{nr^2 - \sqrt{n}}{2}\,.
\end{align*}
The result follows from choice of $r$.
\end{proof}